\documentclass[runningheads]{llncs}
\usepackage{tikz}
\usetikzlibrary{automata, positioning, arrows.meta}
\usepackage{graphicx}
\usepackage{amssymb}
\usepackage{amsmath}

\begin{document}
	
	\title{SMALL UNIVERSAL MULTISET REACTION SYSTEMS\thanks{This study was partially supported by Core Program within the National Research, Development and Innovation Plan 2022-2027, carried out with the support of ANC, project no. 23020101(SIA-PRO), contract no 7N/2022}
}
	
	\author{Annemarie-Beatrix Messner\inst{1} 
		\and Andrei P\u aun\inst{1,2,3} }
	
	\authorrunning{A. Messner and A. P\u aun}
	
	\institute{Faculty of Mathematics and Computer Science, University of Bucharest,\\Bucharest, Romania\\
		\email{annemarie-beatrix.messner@s.unibuc.ro, apaun@fmi.unibuc.ro}\and
		National Institute of Research and Development for Biological Sciences, \\ 296 Independenței Bd.
		District 6, 060031, Bucharest. Romania\\
		\email{andrei.paun@incdsb.ro}\and
		Research Institute for Artificial Intelligence (ICIA), Romanian Academy,\\ Bucharest, Romania.\\
\email{apaun@racai.ro}}
	
	\maketitle
\begin{abstract}
This paper presents a construction of a small universal machine within the framework of Multi-set Reaction Systems. We encode the machine state using elements that represent registers and instruction labels, and we enforce sequential execution by ensuring that reaction chains for a given instruction can only activate after the corresponding label element appears. We implement a universal register machine with 8 registers and 23 instructions, showing that the priority-based model requires 87 distinct elements. The results demonstrate that multiset reaction systems are capable of simulating universal computation.

\end{abstract}

\section{Introduction}

Reaction Systems (RS), introduced by Andrzej Ehrenfeucht and Grzegorz Rozenberg \cite{reaction1,reaction2}, provide a qualitative framework for modeling biochemical and regulatory processes based on the fundamental mechanisms of facilitation and inhibition. In this formalism, the state of the system is represented as a set of entities, and reactions are enabled solely by the presence or absence of these entities. While this abstraction successfully captures the logical structure of regulatory interactions, it deliberately neglects quantitative aspects such as molecular concentrations, multiplicities, and reaction rates.
In living cells, molecular species are typically present in multiple copies, and their abundance plays a decisive role in determining whether reactions occur and at what rate. The assumption that an entity is either present or absent ignores the influence of concentration on system dynamics and prevents the description of many important biological phenomena, including threshold effects, competition for limited resources, and stochastic fluctuations arising from finite molecular populations.

Introducing multisets into the framework naturally overcomes this limitation. A multiset associates each entity with a multiplicity, allowing the state of the system to represent molecular counts rather than mere existence. Reactions can then consume and produce specified numbers of molecules, providing a more faithful representation of biochemical transformations while preserving the structural simplicity of the reaction-system paradigm. This extension bridges the gap between purely qualitative models and quantitative models based on molecular populations.

A further step toward realism consists in enriching reactions with kinetic information. Classical Reaction Systems treat all enabled reactions as occurring simultaneously and without regard to their relative speeds. In contrast, chemical and biochemical reactions proceed at rates determined by intrinsic kinetic constants and by the concentrations of the reacting species. Associating a kinetic parameter with each reaction allows the evolution of the system to reflect not only the logical conditions for reaction occurrence but also the temporal scale on which reactions take place.

This quantitative extension naturally connects the framework to the law of mass action, one of the cornerstones of chemical kinetics. According to the law of mass action, the rate of an elementary reaction is proportional to the product of the concentrations (or molecular counts in a discrete setting) of its reactants, each raised to the power of its stoichiometric coefficient. By combining multisets with reaction kinetics, the enabling of reactions remains governed by the presence of reactants and the absence of inhibitors, as in the original Reaction Systems framework, while the frequency or probability of reaction execution becomes proportional to the corresponding mass-action rate. Consequently, the resulting formalism preserves the regulatory semantics that characterize Reaction Systems while incorporating the quantitative behavior that governs real biochemical processes.

This enriched framework therefore provides a natural bridge between qualitative regulatory models and quantitative kinetic models. It retains the conceptual elegance and compositional nature of Reaction Systems while extending their expressive power to describe concentration-dependent dynamics, stoichiometric constraints, and time-dependent behavior. Such an approach enables the study of biological systems at a level of abstraction that remains mathematically tractable yet is substantially closer to the mechanisms underlying real chemical and biochemical reactions.

We provide in the following a step in this direction by considering the multiset RS and changing the behaviour of the system in the sense that the reaction with most reactants will fire next in a sequential manner.

Multiset reaction systems, formally defined in \cite{multiset}, introduce the use of multisets instead of sets to represent the reactants, inhibitors and products of a reaction. This allows the model to be quantitative, as it can now take into account the quantity of a resource. The permanency principle is modified to allow substances that are not consumed to persist to the next state. This results in a model that is inherently nondeterministic as the reactions will have to compete for the available resources in a state. To model a specific multiset reaction system, we need to define the background set of elements and the set of possible reactions.

Small universal register machines were described  in \cite{korec1996small} and have been used as a standard model of computation to show the Turing completeness of various computational models in the membrane systems area, such as spiking neural P systems \cite{paun2007small} and \cite{Ibarra2009Sequential} or other connected research topics such as \cite{bernardini}, \cite{dfca}, \cite{ionescu}. In the following we will use simiar ideas from the aforementioned papers to the area of reaction systems.
In this paper we will show the implementation of a small universal register machine with 23 instructions using multiset reactions systems using max-spiking manner. 
This involves the modelling of the ADD, SUB and HALT instructions using reactions, as well as encoding the instruction labels and registers as elements in the background set.
We prove that this approach is universal and give an upper bound for the number of components of the system.

\section{Preliminaries}
We first recall the basic notions concerning reaction systems,
introduced in \cite{reaction1,reaction2}. Given a finite and nonempty set $S$, we say that a {\it reaction} in $S$ is a triplet $a=(R,I,P)$, where $R,I,P$ are subsets of $S$ and $R\cap I=\emptyset$. The elements of $S$ are informally called objects, the elements of $R$ are called {\it reactants}, the elements of $I$ are called {\it inhibitors}, while the elements of $P$ are called {\it products}.

A {\it reaction system} in $S$ is an ordered pair ${\cal A} = (S,A)$, where $A$ is a non-empty
set of reactions in $S$.
We now define the conditions for a reaction in $S$ to be enabled. Given 
the reaction system ${\cal A} = (S,A)$ and a subset $X$ of $S$, the reaction  $a = (R_a, I_a, P_a)$ is enabled by $X$, denoted by the boolean function $en_a(X)$, if $R_a \subseteq X$ and
$I_a \cap X = \emptyset$ holds. The result of $a$ on $X$, denoted by $res_a(X)$, is a set defined by $res_a(X) = P_a$ if $en_a(X)$, and $res_a(X) = \emptyset$, otherwise.
Clearly, $X\subseteq S$ may enable more than one reaction, and this is captured by the following definition. Let ${\cal A} = (S,A)$
be a reaction system with background set $S$ and let $X \subseteq S$. The subset of
reactions of $A$ enabled by $X$, denoted by $en(A,X)$, is defined as 
$en(A,X) =\{a \in A | en_a(X)\}$, and the result of $A$ on $X$, denoted by $res(A,X)$, is $res(A,X) = \displaystyle{\bigcup_{a\in en(A,X)} res_a(X)}$.
As one can see, the effect of a set of reactions enabled by a subset of $S$ is cumulative. In what follows, we say that $X\subseteq S$ is a state of the system.

The dynamics of reaction systems is determined by the interactions of successive reactions enabled by successive states. Thus, the system starts in an initial state enabling a set of reactions that produce a result which will be enhanced by some context and form the new state.
More precisely, the dynamics of a reaction system is described by an {\it interactive process} as follows.
Let  ${\cal A} = (S,A)$ be a reaction system,  an interactive process in 
${\cal A}$ is a pair of two finite sequences,
$\pi = (\alpha, \beta)$, where $\alpha = C_0,C_1, \dots,C_n$ is the {\it context sequence}, $\beta = D_1, D_2,\dots,D_n$ is the {\it result sequence}, with $n \ge 1$, and $C_i, D_j \subseteq S$, $0\le i\le n$,
$1\le j\le n$. The components of the context sequence are given and they represent the influence of the environment on the computation, while the components of the result sequence are computed by $D_1 = res(A,C_0)$, 
and $D_i = res(A,D_{i-1} \cup C_{i-1})$ for each $2 \le i \le n$.

\section{Multiset Reaction Systems}

In this model, a reaction is formally defined as a triple $(R, I, P)$ where $R$ is the \textit{set of reactants} $I$ is the \textit{set of inhibitors}, and $P$ is the \textit{set of products}.
Two fundamental principles govern the execution of reactions: \textit{activation} and \textit{inhibition}. 
A reaction is enabled when all of its reactants are present in the current state, but none if its inhibitors are. 
As a result, the enabled reaction produces a set of products in the next state that could either activate or inhibit other reactions.

The original reaction systems model is considered a qualitative model of biochemical interactions, as it does not take into account the quantity of a resource. 
If a reactant is available in a state, it is available in an infinite quantity. 
On the other hand, if an element is consumed in a reaction, then it is completely removed from the state.
In this model, elements cannot persist between states if they are not produced by a reaction. 

Furthermore, the original reaction systems model is a computation state-transition system in which successor states are generated by previously produced reactants and input from the environment. 
This makes the model \textit{context-dependent}. A model that does not consider the contribution of the environment is known as \textit{context-independent}.

Multiset Reaction Systems are a context-independent extension of the original reaction systems model that incorporates the quantity of a resource. 
The model is defined in \cite{multiset} and is based on replacing the sets of reactants, inhibitors, and products with multisets.
This allows the modification of the permanency principle. 
In this model, resources can persist between states if they are not entirely consumed by reactions. 
As each resource is available in a specific quantity, a constraint is effectively placed on the number of reactions that can be executed in a state. 
This results in an inherently quantitative and nondeterministic model. Now we will give the definitions of a few concepts related to multiset reaction systems.

Multiset reaction systems are formally defined in \cite{multiset} as a pair $\mathcal{A} = (S, A)$ where $S$ is a finite set of elements called the \textit{background set} and $A$ is a finite set of reactions.
$S$ is a finite set of at least two elements and its components are informally known as \textit{{objects}}. 
The elements of $A$ are called \textit{reactions} and are defined as triplets $(R, I, P)$ where $R$, $I$, and $P$ are multisets over $S$ representing the \textit{reactants}, \textit{inhibitors}, and \textit{products} of the reaction, respectively. 
$R$ and $I$ are disjoint multisets, meaning that they cannot share any common elements. We denote as $MRS(S)$ the set of all multiset
reaction systems over the background $S$. 

Multisets extend the definition of sets by allowing each element to be accompanied by a multiplicity, which indicates how many times that element appears in the multiset. 
In order to compute the number of copies for each element, we employ a \textit{membership function} $\sigma_X : U \to \mathbb{N}$, where $\sigma_X(x)$ denotes how many times the element $x$ from a finite set $U$ appears in the multiset $X$.
An \textit{empty multiset} is a multiset in which all elements have a multiplicity of zero and is known as $\epsilon$. 
The \textit{weight} of a multiset $X$ over $U$ is defined as:
 \[\lVert X\rVert = \sum_{a \in U} \sigma_X(a). \]

A reaction $a = (R, I, P)$ is \textit{enabled (activated)} by a state $T$ if $R$ is a submultiset of $T$ and $I$ is disjoint from $T$. 
When a reaction is enabled, it produces the multiset of products $P$ in the next state. 
We note whether a reaction $a$ is enabled by a state $T$ as $\mathrm{en}_a(T)$.
A key difference between the original reactions systems model and the multiset reaction systems model is that elements that are not removed as a result of a reaction can persist in the next state if not consumed.
This helps us define $\mathrm{res}_a(T)$, the result of executing a reaction $a$ in a state $T$ as:
\begin{equation*}
    \mathrm{res}_a(T) = \begin{cases}
        P + (T \setminus R) & \text{, if } \mathrm{en}_a(T) \\
        \emptyset & \text{, otherwise}
    \end{cases}
\end{equation*}
For a multiset reaction $\Delta$ in a multiset reactions system $\mathcal{A} = (S, A)$, we define $R_\Delta$, $I_\Delta$ as follows:
\begin{align*}
    R_\Delta &= \sum_{a \in A} \sigma_\Delta(a) \,\cdot\, R_a \\
        I_\Delta &= \bigcup_{\substack{a \in A \\
                                      \sigma_\Delta(a) \ge 1}} I_a
\end{align*}

The multiset reaction systems model defines a state transition system over $S$ with a computation path $T_0, T_1, T_2, ...$ where $T_{i+1} = \bigcup_{a \in A} \mathrm{res}_a(T_i)$ for all $i \geq 0$.
We consider that a state $T$ enables simultaneously the reactions of $\Delta$, noted as $\mathrm{en}_\Delta(T)$, if $R_\Delta$ is a submultiset of $T$ and $I_\Delta$ is disjoint from $T$.
The result of executing a multiset of reactions is defined as:
\begin{equation*}
    \mathrm{res}_\Delta(T) = 
        \sum_{a \in A} \sigma_\Delta(a) \,\cdot\, P_a + (T \setminus R_\Delta) \\
\end{equation*}

If there is a multiset of reactions $\Delta$ of weight $k$, such that $\mathrm{en}_\Delta(T)$ is True, then we call that a \textit{k-event}.
A \textit{1-event} is also known as a \textit{sequential event}.

An \textit{f-evolution} generated by $T$ is a sequence of multisets over S: $D_0, D_1, D_2, ...$ such that $D_0 = T$ and for all $i \geq 0$, the multiset $D_{i+1}$ is generated from $D_i$ by and \textit{f-event} such that: $D_{i+1} = \mathrm{res}_{\Delta_i}(D_i)$.

The computation of the machine $M_a\in MRS(S)$ where $a\in S$ is given by the multiplicity of the element $a$ in the halting position of the machine $M_a$. If the machine does not halt we do not consider that computation to genereate any output. Thus $L(M_a)=\{\sigma_X(a)\mid X\mbox{ is the halting state of the machine}\}$.

For computing multiset reaction systems, a special \textit{universal inhibitor}, noted as $\dagger$, is introduced and considered to be a part of the inhibitor set of every reactions.
Consequently, if $\dagger$ is present in a state, then no reaction can be enabled and the system halts.

\subsection{The register machines}
The register machine is a well-known computational model, composed of a set of registers and a finite set of instructions. 
Each of these register can store a non-negative integer value.
The machine operates by manipulating these values through a sequence of arithmetic operations and control flow branches.

Formally, the register machine is defined as $M = (m, H, L_0, L_h, I)$ where: 
\begin{itemize}
    \item $m$ is the number of registers, which are denoted as $R_0, R_1, ..., R_{m-1}$,
    \item $H$ is the finite set of instruction labels,
    \item $L_0 \in H$ is the initial instruction label,
    \item $L_h \in H$ is the halting instruction label,
    \item $I$ is a finite set of instructions, that are bijectively mapped to $H \setminus \{L_h\}$.
\end{itemize}

A register machine $M$ starts with all registers initialized to zero and proceeds to generate a set of number. 
First the instruction labeled $L_0$ is applied, and then the machine proceeds to execute instructions as defined by the labels.
If the instruction $L_h$ is reached, the machine halts and the final state of the register $R_0$ is considered the output of the computation.
If the machine does not halt, then no output is generated. 
All the numbers generated by a register machine $M$ compose the set denoted as $\mathcal{N}(M)$. 
All the numbers generated by a register machine are known to be Turing computable \cite{minsky1967}.

All functions that can be computed by a Turing machine can also be computed by register machines. 
For this we need to introduce the arguments in specific registers: $R_1, R_2, ..., R_k$ and then start the execution by running the instruction with the label $L_0$.
When the instruction labeled $L_h$ is reached and the machine halts, then the value of the function can be found in the register $R_0$.

In \cite{korec1996small}, the register machines are used for computing functions.
In this context, a register machine is defined as universal if there is a recursive function $g$ such that for all natural numbers $x$, $y$ we have $\phi_x(y) = M_u(g(x), y)$, where ($\phi_0, \phi_1, ...$) are a fixed admissible enumaration of the unary partial recursive functions \cite{korec1996small}.

Several strongly universal register machines are constructed in \cite{korec1996small}, having the the numbers $g(x)$ and $y$ introduced as input in registers $R_1$ and $R_2$.
The result is obtained in register $R_0$.

\section {Small universal systems}

Multiset Reaction Systems are universal in the max-sequentiality mode. To prove this we will show that any Turing Machine working of numbers  can be simulated by a MRS.

\begin{theorem}
Multiset Reaction Systems are universal in the max-sequentiality mode. 
\end{theorem}
\proof
The proof proceeds by simulating the set $T$ generated by the TM through a register machine, thus we start by encoding the registers and instructions labels as elements in the background set and then constructing the instruction set for a small non-deterministic universal machine as chains of reactions, effectively showing how each instruction can be represented within the multiset reaction system formalism.
 Given the register machine we consider the element $R_0$ as the encoding of the register 0, thus we will prove that there is $M_{R_0}\in MRS(S)$ where $R_0\in S$ that for the element $n\in T$ will have the value $2n$ in the halting computation of the machine $M_{R_0}$.

The instructions we consider are:

\begin{description}
    \item[$L_i$]: $\mathrm{ADD}(R_i, L_j, L_k)$ - Add $1$ to the value stored in register $R_i$, then jump non-deterministically to either instruction with label $L_j$ or to instruction with label $L_k$.
    \item[$L_i$]: $\mathrm{SUB}(R_i, L_j, L_k)$ - If the value stored in register $R_i$ is not zero, then subtract $1$ from $R_i$ and jump to the instruction with label $L_j$, otherwise jump to the instruction with label $L_k$.
    \item[$L_h$]: $\mathrm{HALT}$ - the machine stops execution.
\end{description}

Starting from a register machine ${M(m, H, L_0, L_h, I)}$ we will construct a multiset reaction system $\mathcal{A}$ = $(S, A)$ that simulates the behavior of $M$. 
The background set $S$ will contain elements modelling the state of the registers, the instruction labels, and auxiliary elements needed for the simulation. 
The set of reactions $A$ will include reactions that form each type of instruction in the register machine.
We associate to each register in $M$ an element in $S$ such as the number of copies of that element in a state corresponds to the value stored in the register. 
All registers, denoted by ($0,1,2,...,m-1$), will have a corresponding element $R_r$ such that $r = 0, 1, ..., m-1$. 
If the value stored in register $R_r$ is $k$, then the state of the multiset reaction system will contain $k$ copies of the element $R_r$.
The output register is register 0, therefore the number computed by the machine will be the number of copies of $R_0$ in the state when the machine halts.
We also associate to each instruction label $L_i$ from $M$ an element $L_i$ in $S$ such that the presence of that element in a state will trigger the equivalent set of reactions for that label to be executed.

Reactions enabled within an instruction will be noted as $a_{ij}$ (for ADD instructions) and $s_{ij}$ (for SUB instructions) where $i$ corresponds to the instruction label and $j$ is an index to distinguish between different reactions within the same instruction.
For the halt instruction we will simply note the specific reactions used in the implementation with $h_i$, where $i$ is used to define between different reactions.

Finally, we introduce auxiliary elements in $S$ to help with the simulation of the instructions. 
Such elements will be noted as $Z_{ij}$ where $i$ coresponds to the instruction label and $j$ is an index to distinguish between different auxiliary elements for the same instruction.
We will note the auxiliary elements used in the implementation of the halt instruction with $H_i$ where $i$ is used to define between different reactants.

We propose a setting that will be able to do  the simulation of the register machines:  imposing sequentiality for the execution of reactions via a priority rule in the rection systems.

This approach was inspired by the max-sequentiality spiking of neural P systems \cite{Ionescu2006Spiking} . 
The \textit{max-sequentiality} condition for spiking neural P Systems described in \cite{Ibarra2009Sequential} states that at every step, the systems chooses only one neuron that can fire from all active neurons, based on the maximum number of spikes stored. 
Similarly our multiset reaction system will execute only one reaction at the time and the reaction to be executed will be chosen from all currently enabled reactions by the maximum number of reactants. 

The sequential SNP system encodes the number $n$ stored within a register as $2n + 3$ spikes within a coresponding neuron. 
We will use a similar encoding for the number stored in a register. Therefore, if the number $n$ is stored in register $R_r$, then the state of the multiset reaction system will contain $2n + 1$ copies of the element $R_r$.

Using these rules, we will now define the ADD, SUB and HALT instructions for small universal machines:

\subsection{The ADD Instruction}
We define the ADD instruction as follows: $L_i$: $\mathrm{ADD}(R_r, L_j, L_k)$.
This instructions adds $1$ to the value stored in register $R_r$, then jumps non-deterministically to either the instruction with label $L_j$ or to the instruction with label $L_k$.
For this instruction we define the following background set of elements and reactions:
\begin{align*}
    S_{ADD} &= \{L_i, L_j, L_k, R_r\} \\
    A_{ADD} &= \{a_{i1}, a_{i2}\} \\
    a_{i1} &= (\{L_i\}, \{\dagger\}, \{R_r, R_r, L_j\}), \\
    a_{i2} &= (\{L_i\}, \{\dagger\}, \{R_r, R_r, L_k\})
\end{align*}
The presence of $L_i$ in a state will trigger the execution of the ADD instruction, which will produce two copies of the element $R_r$ and either $L_j$ or $L_k$ depending on the reaction that is executed.
Generating two copies of $R_r$ effectively adds $1$ to the value stored in register $R_r$ as we are using the encoding where the number $n$ is represented by $2n + 1$ copies of the element $R_r$.
In multiset reaction systems, each resource is available in a specific quantity, introducing a constraint on the execution of reactions. 
Both reactions have the same number of reactants, but only one of the two reactions enabled by the presence of $L_i$ can be executed, so the choice of which reaction to execute is non-deterministic. 
The execution flow of the reactions can be visualized in Figure \ref{fig:add_priority_state_transition}.
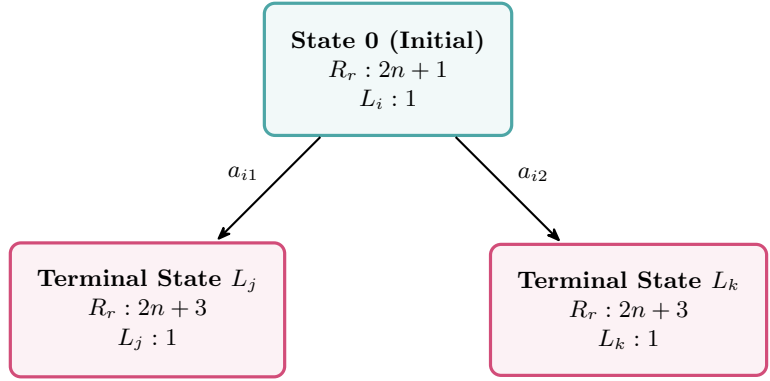
\begin{figure}[h]
    \centering
    \begin{tikzpicture}[
        shorten >=1pt, 
        node distance=4.5cm, 
        on grid, 
        auto,
        % Style for State Boxes
        state_style/.style={
            rectangle, 
            draw=teal!70, 
            fill=teal!5, 
            very thick, 
            rounded corners, 
            inner sep=10pt, 
            align=center,
            font=\small
        },
        % Style for Final Out-State
        final_style/.style={
            rectangle, 
            draw=purple!70, 
            fill=purple!5, 
            very thick, 
            rounded corners, 
            inner sep=10pt, 
            align=center,
            font=\small
        },
        >={Stealth[round]},
        thick
    ]

        % Initial State
        \node[state_style] (S0) {
            \textbf{State 0 (Initial)} \\ 
            $R_r: 2n + 1$ \\ 
            $L_i: 1$
        };

        % Upper Branch: State after reaction a1
        \node[final_style] (S1) [below left=of S0] {
            \textbf{Terminal State $L_j$} \\ 
            $R_r: 2n + 3$ \\ 
            $L_j: 1$
        };

        % Lower Branch: State after reaction a2
        \node[final_style] (S2) [below right=of S0] {
            \textbf{Terminal State $L_k$} \\ 
            $R_r: 2n + 3$ \\ 
            $L_k: 1$
        };

        % Connective Transitions
        \path[->] 
            (S0) edge[left=15]  node[above left]  {$a_{i1}$} (S1)
            (S0) edge[right=15] node[above right] {$a_{i2}$} (S2);

    \end{tikzpicture}
    \caption{The ADD module with priority-based sequential reaction execution.}
    \label{fig:add_priority_state_transition}
\end{figure}

\subsection{The SUB Instruction}
The SUB instruction is defined as follows: $L_i$: $\mathrm{SUB}(R_r, L_j, L_k)$. 
This instruction checks if the value stored in register $R_r$ is not zero, then subtracts $1$ from $R_r$ and jumps to the instruction with label $L_j$.
Otherwise the value in $R_r$ remains zero and the program jumps to the instruction with label $L_k$.
For this instruction we define the following background set of elements and reactions:
\begin{align*}
    S_{SUB} &= \{L_i, L_j, L_k, R_r, Z_{i1}, Z_{i2}, Z_{i3}, Z_{i4}\} \\
    A_{SUB} &= \{s_{i1}, s_{i2}, s_{i3}, s_{i4}, s_{i5}, s_{i6}, s_{i7}\} \\
    s_{i1} &= (\{L_i\}, \{\dagger\}, \{Z_{i1}, Z_{i1}, Z_{i3}\}), \\
    s_{i2} &= (\{R_r, R_r, R_r, Z_{i3}\}, \{\dagger\}, \{Z_{i1}, Z_{i2}, R_r\})\\
    s_{i3} &= (\{Z_{i1}, Z_{i1}\}, \{\dagger\}, \{R_r, R_r\}), \\
    s_{i4} &= (\{Z_{i1}, Z_{i1}, Z_{i1}\}, \{\dagger\}, \{Z_{i2}, Z_{i4}\}), \\
    s_{i5} &= (\{Z_{i2}, Z_{i2}\}, \{\dagger\}, \emptyset), \\
    s_{i6} &= (\{Z_{i4}\}, \{\dagger\}, \{L_j\}), \\
    s_{i7} &= (\{Z_{i1}, Z_{i2}\}, \{\dagger\}, \{L_k\})
\end{align*}

The execution flow of the SUB instruction can be examined in Figure \ref{fig:sub_priority_state_transition}.
The first reaction enabled is $s_{i1}$ as it is the only reaction enabled by the presence of $L_i$ in the state.
The activation of $s_{i1}$ produces two copies of $Z_{i1}$ and one copy of $Z_{i3}$, which enables the reactions $s_{i2}$ and $s_{i3}$.
Reactions $s_{i2}$ and $s_{i3}$ have 4 and 2 reactants respectively, so $s_{i2}$ is the reaction that will be executed as it has the highest number of reactants.
However, reaction $s_{i2}$ can only be executed if there are at least 3 copies of $R_r$ in the state, which corresponds to the value stored in register $R_r$ being at least 1. 
This is where the execution of the reactions chain branches based on the value stored in $R_r$. 

First, we will examine the case where the value stored in $R_r$ is at least 1. 
In this case, $s_{i2}$ is enabled and will be executed first, producing: a copy of the element $Z_{i1}$, a copy of the element $Z_{i2}$, and a copy of the element $R_r$; and consumes 3 copies of $R_r$ and one copy of $Z_{i3}$.
After this state transition, the number of copies of $R_r$ is reduced by 2, which corresponds to the subtraction of 1 from the value stored in register $R_r$.
Additionally, there are no more copies of $Z_{i3}$ in the state, which means that $s_{i2}$ is no longer enabled, and the only enabled reactions are $s_{i3}$, $s_{i4}$, and $s_{i7}$.
Reaction $s_{i4}$ has the highest number of reactants, so it will be executed next, producing a copy of $Z_{i2}$, a copy of $Z_{i4}$, and consuming 3 copies of $Z_{i1}$.
After executing $s_{i4}$, the only enabled reactions are $s_{i5}$ and $s_{i6}$, which require 2 and 1 reactants respectively.
Naturally $s_{i5}$ will be activated next, consuming 2 copies of $Z_{i2}$ and producing nothing, which leaves $s_{i6}$ as the only enabled reaction.
The execution of $s_{i6}$ produces a copy of $L_j$, which corresponds to the jump to instruction $L_j$ in the register machine.

Now we will examine the case where the value stored in $R_r$ is zero.
In this case, the number of copies of the element $R_r$ is 1, so reaction $s_{i2}$ is not enabled and that leaves $s_{i3}$ as the only reaction that can be activated.
After the execution of $s_{i3}$, there are no more copies of $Z_{i1}$ in the state, which means that $s_{i3}$ cannot be enabled again. 
Additionally, two more copies of $R_r$ are produced, which means that $s_{i2}$ is now enabled. 
The execution of $s_{i2}$ produces a copy of $Z_{i1}$, a copy of $Z_{i2}$, and returns the number of copies of $R_r$ to their original count.
The only enabled reaction now is $s_{i7}$, which produces a copy of $L_k$, corresponding to the jump to instruction $L_k$ in the register machine.

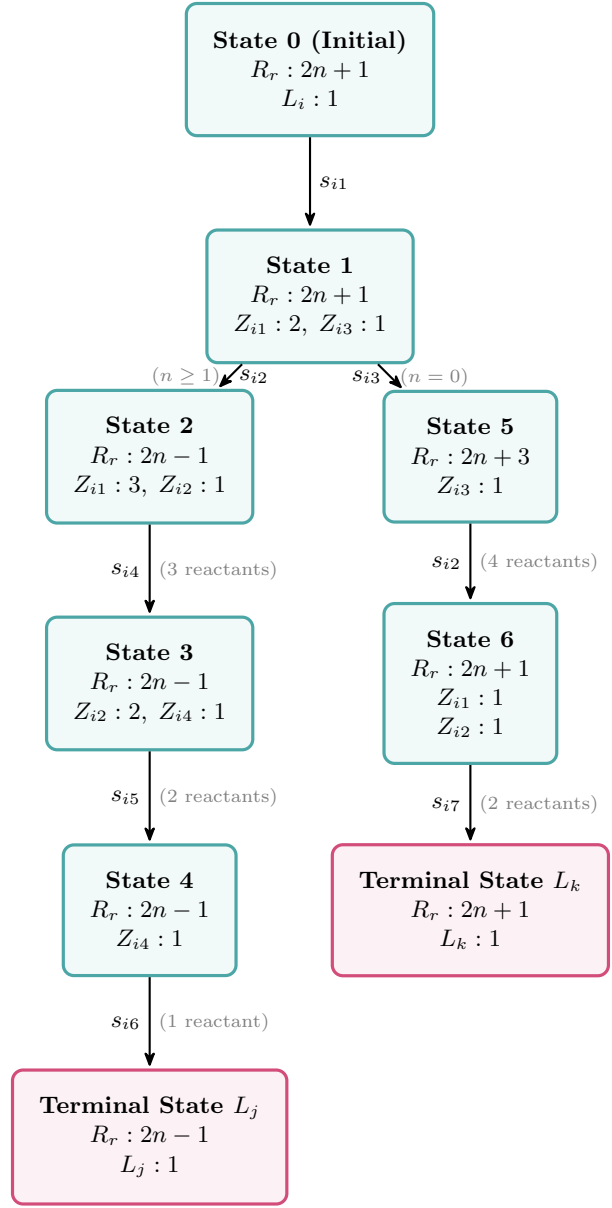
\begin{figure}[!ht]
    \centering
    \begin{tikzpicture}[
        shorten >=1pt, 
        node distance=3cm, 
        on grid, 
        auto,
        % Style for State Boxes
        state_style/.style={
            rectangle, 
            draw=teal!70, 
            fill=teal!5, 
            very thick, 
            rounded corners, 
            inner sep=10pt, 
            align=center,
            font=\small
        },
        % Style for Final Out-State
        final_style/.style={
            rectangle, 
            draw=purple!70, 
            fill=purple!5, 
            very thick, 
            rounded corners, 
            inner sep=10pt, 
            align=center,
            font=\small
        },
        >={Stealth[round]},
        thick
    ]

        % State 0: Start
        \node[state_style] (S0) {
            \textbf{State 0 (Initial)} \\ 
            $R_r: 2n + 1$ \\ 
            $L_i: 1$
        };

        % State 1: Only s1 can fire (requires Li)
        \node[state_style] (S1) [below=of S0] {
            \textbf{State 1} \\ 
            $R_r: 2n + 1$ \\ 
            $Z_{i1}: 2, \; Z_{i3}: 1$
        };

        % State 2: s3 takes priority over s2 (s2 is no longer enabled)
        \node[state_style] (S2) [below left=of S1] {
            \textbf{State 2} \\ 
            $R_r: 2n - 1$ \\ 
            $Z_{i1}: 3, \; Z_{i2}: 1$
        };

        \node[state_style] (S3) [below=of S2] {
            \textbf{State 3} \\ 
            $R_r: 2n - 1$ \\ 
            $Z_{i2}: 2, \; Z_{i4}: 1$
        };

        \node[state_style] (S4) [below=of S3] {
            \textbf{State 4} \\ 
            $R_r: 2n - 1$ \\ 
            $Z_{i4}: 1$
        };

        % State 3: s7 takes priority over s3 and s5 (2 reactants vs 1 or 2 with lower weight)
        \node[final_style] (S7) [below=of S4] {
            \textbf{Terminal State $L_j$} \\ 
            $R_r: 2n - 1$ \\ 
            $L_j: 1$
        };

        % State 2: s2 takes priority over s3 (4 reactants vs 1)
        \node[state_style] (S5) [below right=of S1] {
            \textbf{State 5} \\ 
            $R_r: 2n + 3$ \\ 
            $Z_{i3}: 1$
        };

        % State 2: s2 takes priority over s3 (4 reactants vs 1)
        \node[state_style] (S6) [below=of S5] {
            \textbf{State 6} \\ 
            $R_r: 2n + 1$  \\
            $Z_{i1}: 1$ \\
            $Z_{i2}: 1$ 
        };

        \node[final_style] (S8) [below=of S6] {
            \textbf{Terminal State $L_k$} \\ 
            $R_r: 2n + 1$ \\ 
            $L_k: 1$
        };

        % Linear Transitions showing priority choice
        \path[->] 
            (S0) edge node[right] {$s_{i1}$} 
                       node[right, font=\scriptsize, text=gray] {} (S1)
            (S1) edge node[right] {$s_{i2}$} 
                       node[left, font=\scriptsize, text=gray] {($n \ge 1$)} (S2)
            (S2) edge node[left] {$s_{i4}$} 
                       node[right , font=\scriptsize, text=gray] {(3 reactants)} (S3)
            (S3) edge node[left] {$s_{i5}$} 
                       node[right, font=\scriptsize, text=gray] {(2 reactants)} (S4)
            (S4) edge node[left] {$s_{i6}$} 
                       node[right, font=\scriptsize, text=gray] {(1 reactant)} (S7)
            (S1) edge node[left] {$s_{i3}$} 
                       node[right, font=\scriptsize, text=gray] {($n = 0$)} (S5)
            (S5) edge node[left] {$s_{i2}$} 
                       node[right, font=\scriptsize, text=gray] {(4 reactants)} (S6)
            (S6) edge node[left] {$s_{i7}$} 
                       node[right, font=\scriptsize, text=gray] {(2 reactants)} (S8);

    \end{tikzpicture}
    \caption{The SUB module with priority-based sequential reaction execution.}
    \label{fig:sub_priority_state_transition}
\end{figure}

\subsection{The HALT Instruction}
The HALT instruction is defined as $L_h$: $\mathrm{HALT}$. 
When $L_h$ is fired we need to compute the exact number $n$ stored in the register $R_0$ and then produce a copy of the element $\dagger$ to stop all possible reactions.
To accomplish this, we define the following background set of elements and reactions:
\begin{align*}
    S_{HALT} &= \{L_h, R_0, H_1, H_2\} \\
    A_{HALT} &= \{h_1, h_2, h_3, h_4\} \\
    h_1 &= (\{L_h, R_0, R_0\}, \{\dagger\}, \{L_h, H_{1}\}), \\
    h_2 &= (\{L_h, R_0\}, \{\dagger\}, \{H_2\}) \\
    h_3 &= (\{H_1, H_2\}, \{\dagger\}, \{R_0, H_2\}), \\
    h_4 &= (\{H_2\}, \{\dagger\}, \{\dagger\}), \\
\end{align*}

Before the execution of the HALT instruction, the state contains $2n + 1$ copies of $R_0$.
The existence of $L_h$ in the state will enable the execution of reactions $h_1$ and $h_2$.
If $n \ge 1$, then $h_1$ is the reaction with the highest number of reactants, so it will be executed first, producing a copy of $H_1$ and consuming two copies of $R_0$.
As long as there are enough copies of $R_0$ in the state, we can keep executing reaction $h_1$, effectively producing $n$ copies of $H_1$ and consuming $2n$ copies of $R_0$.
This leaves us with only one copy of $R_0$ and one copy of $L_h$ in the state.
Now the only enabled reaction is $h_2$, which produces a copy of $H_2$ and consumes the last copy of $R_0$ and the last copy of $L_h$.
In this next state, both $h_3$ and $h_4$ are enabled, but $h_3$ has more reactants, so it will be executed first.
We will now proceed to execute reaction $h_3$ and restore the value $n$ in the register $R_0$.
Activating $h_3$ produces a copy of $R_0$ and consumes one copy of $H_1$ and one copy of $H_2$.
The element $H_2$ is both consumed and produced by $h_3$, ensuring that all reactions happen sequentially. 
When there are no more copies of $H_1$ in the state, the only enabled reaction is $h_4$, which produces a copy of the universal inhibitor and consumes the last copy of $H_2$, effectively halting the system.

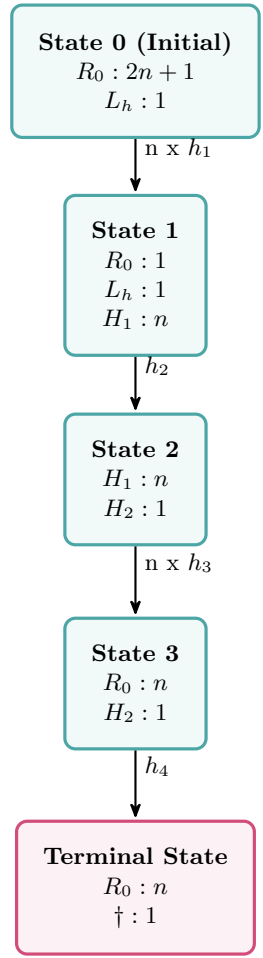
\begin{figure}[!ht]
    \centering
    \begin{tikzpicture}[
        shorten >=1pt, 
        node distance=2.7cm, 
        on grid, 
        auto,
        % Style for State Boxes
        state_style/.style={
            rectangle, 
            draw=teal!70, 
            fill=teal!5, 
            very thick, 
            rounded corners, 
            inner sep=10pt, 
            align=center,
            font=\small
        },
        % Style for Final Out-State
        final_style/.style={
            rectangle, 
            draw=purple!70, 
            fill=purple!5, 
            very thick, 
            rounded corners, 
            inner sep=10pt, 
            align=center,
            font=\small
        },
        >={Stealth[round]},
        thick
    ]

        % Initial State
        \node[state_style] (S0) {
            \textbf{State 0 (Initial)} \\ 
            $R_0: 2n + 1$ \\ 
            $L_h: 1$
        };

        \node[state_style] (S1) [below=of S0] {
            \textbf{State 1} \\ 
            $R_0: 1$ \\ 
            $L_h: 1$ \\
            $H_1: n$
        };

        \node[state_style] (S2) [below=of S1] {
            \textbf{State 2} \\ 
            $H_1: n$ \\ 
            $H_2: 1$
        };

         \node[state_style] (S3) [below=of S2] {
            \textbf{State 3} \\ 
            $R_0: n$ \\ 
            $H_2: 1$
        };

        \node[final_style] (S4) [below=of S3] {
            \textbf{Terminal State} \\ 
            $R_0: n$ \\ 
            $\dagger: 1$
        };

        % Connective Transitions
        \path[->] 
            (S0) edge[left=15]  node[above right]  {n x $h_1$} (S1)
            (S1) edge[right=15] node[above right] {$h_2$} (S2)
            (S2) edge[right=15] node[above right] {n x $h_3$} (S3)
            (S3) edge[right=15] node[above right] {$h_4$} (S4);

    \end{tikzpicture}
    \caption{The HALT module with priority-based sequential reactions execution.}
    \label{fig:halt_priority_state_transition}
\end{figure}

\section{Conclusions}

We introduced the implementation of a small universal machine of  23 instructions using multiset reactions systems.
We encoded the registers and instructions labels as specific elements. 
We have shown that all types of instructions can be modeled using reaction chains.
Reactions chains that model a specific instruction can only be enabled once the element encoding the instruction label is generated, ensuring that instructions are executed sequentially.
%We designed the instructions as isolated sets of reactions using two approaches: using priority-based sequentiality and using inhibitors.

The  approach considers that reactions are executed sequentially based on a priority order defined by the number of reactants (max sequentiality). 
This means that for all enabled reactions in a state, the reaction with the most number of reactants will be chosen to activate. 
We have concluded that 87 distinct elements are needed to build the small universal machine using this approach.

%The second approach uses inhibitors to control the flow of reactions.  In this case, the order of execution is enforced by the availability of the reactants in a specific quantity and the absence of specific inhibitors. A one-state waiting phase is introduced between the execution of two reactions to test the existance of a specific element. Building the small universal machine using this approach requires 83 distinct elements.

Further research could consider the use of other approaches to enforce sequentiality, namely the variation in execution time or using inhibitors to control the flow of reactions.
Moreover, the number of elements used could be reduced by optimizing the number of auxiliary elements needed to build the reaction chains, or by implementing a different small universal machine model with fewer instructions. All these results may be applicable also to models based on the work from \cite{cheruku}, \cite{blind} or \cite{numerical} which are interesting avenues of further research.

\end{document}